\documentclass[letterpaper, 10 pt, conference, natbib]{ieeeconf} 

\IEEEoverridecommandlockouts                              %

\makeatletter
\let\NAT@parse\relax
\makeatother
\usepackage[numbers]{natbib}

\usepackage[bookmarks=true]{hyperref}

\usepackage{hyperref}
\usepackage[utf8]{inputenc}
\usepackage{dblfloatfix} 
\usepackage[T1]{fontenc}    %
\usepackage{hyperref}       %
\usepackage{url}            %
\usepackage{booktabs}       %
\usepackage{amsfonts}       %
\usepackage{nicefrac}       %
\usepackage{microtype}      %
\usepackage{amsmath,amsthm}
\usepackage{subcaption}
\usepackage{wrapfig,graphicx}
\usepackage{color, soul}
\usepackage{comment}
\usepackage{mathtools}
\usepackage{overpic}
\usepackage{xspace}
\let\labelindent\relax
\usepackage{enumitem}
\usepackage{svg}
\usepackage{cleveref}
\usepackage{float}
\usepackage[linesnumbered,ruled,noend]{algorithm2e}
\usepackage[font=small,labelfont=bf]{caption}
\usepackage [english]{babel}
\usepackage [autostyle, english = american]{csquotes}
\MakeOuterQuote{"}

\usepackage{pgfplots}
\pgfplotsset{compat=1.17}
\usepackage{graphics} %
\usepackage{epsfig} %
\usepackage{amssymb}  %
\usepackage{wrapfig}
\usepackage{authblk}

\definecolor{gray}{rgb}{0.35,0.35,0.35}
\definecolor{blue}{rgb}{0,0,1}
\definecolor{red}{rgb}{1,0,0}
\definecolor{orange}{rgb}{0.75, 0.4, 0}
\definecolor{green}{rgb}{0.0, 0.5, 0.0}

\newtheorem*{proposition*}{proposition}

\newcommand{\ignore}[1]{}

 \def\A{\mathcal{A}}

\newcommand{\Cpp}{C\raise.08ex\hbox{\tt ++}\xspace}

\newtheorem*{problem*}{Problem} 

\newtheorem{theorem}{Theorem}
\newtheorem{corollary}{Corollary}

\newtheorem{prop}{Proposition}

\theoremstyle{definition}

\theoremstyle{plain}
\newtheorem{observation}{Observation}

\def\0{\bm{0}}

\makeatletter
\def\thmhead@plain#1#2#3{%
  \thmname{#1}\thmnumber{\@ifnotempty{#1}{ }\@upn{#2}}%
  \thmnote{ {\the\thm@notefont#3}}}
\let\thmhead\thmhead@plain
\makeatother

\newcommand{\storage}{\ensuremath{W}\xspace}

\newcommand{\cols}{\ensuremath{c}\xspace}
\newcommand{\rows}{\ensuremath{r}\xspace}

\newcommand{\nobjs}{\ensuremath{n}\xspace}
\newcommand{\arr}{\ensuremath{\A}\xspace}

\newcommand{\inc}{\ensuremath{A}\xspace}
\newcommand{\dep}{\ensuremath{D}\xspace}
\newcommand{\look}{\ensuremath{\ell}\xspace}

\usepackage{ifthen}
\newboolean{isfull} 
\setboolean{isfull}{false} %

\begin{document}

\title{ \LARGE \bf
Fully Packed and Ready to Go: High-Density,\\ Rearrangement-Free, Grid-Based Storage and Retrieval
}

\author{Tzvika Geft$^{*}$, Kostas Bekris$^{*}$, and Jingjin Yu$^{*}$%
\thanks{$^{*}$Computer Science Dept., Rutgers University, New Brunswick NJ, USA}
}

\maketitle

\begin{abstract}
Grid-based storage systems with uniformly shaped \emph{loads} (e.g., containers, pallets, totes) are commonplace in logistics, industrial, and transportation domains. 
A key performance metric for such systems is the maximization of space utilization, which requires some loads to be placed behind or below others, preventing direct access to them. Consequently, dense storage settings bring up the challenge of determining how to place loads while minimizing costly rearrangement efforts necessary during retrieval. 
This paper considers the setting involving an inbound phase, during which loads arrive, followed by an outbound phase, during which loads depart. The setting is prevalent in distribution centers, automated parking garages, and container ports. In both phases, minimizing the number of rearrangement actions results in more optimal (e.g., fast, energy-efficient, etc.) operations. In contrast to previous work focusing on stack-based systems, this effort examines the case where loads can be freely moved along the grid, e.g., by a mobile robot, expanding the range of possible motions. %
We establish that for a range of scenarios, such as having limited prior knowledge of the loads' arrival sequences or grids with a narrow opening, a (best possible) rearrangement-free solution always exists, including when the loads fill the grid to its capacity.
In particular, when the sequences are fully known, we establish an intriguing characterization showing that rearrangement can always be avoided if and only if the open side of the grid (used to access the storage) is at least 3 cells wide.
We further discuss useful practical implications of our solutions.

\end{abstract}

\IEEEpeerreviewmaketitle

\section{Introduction}

The past two decades have witnessed the dramatic rise of robotics and other automation technologies for transporting uniform-sized \emph{loads} or {\em items} (e.g., containers, pallets, totes, etc.) in logistics applications. Notable examples range from thousands of mobile robots roaming in a single warehouse helping order fulfillment \cite{wurman2008coordinating} to automated cranes and AGVs transporting containers at shipping yards \cite{auto-port}. 
Operations at logistics hubs typically occur in two distinct phases: first, the storage of incoming items, and later, their retrieval for further transport.
Such operations are common in facilities in which goods have to change transport mode or vehicle, with only a limited storage time in between.
In addition to the previous examples, this scenario also arises in cross-docking~\cite{cross-dock}, where goods delivered by incoming trucks are quickly sorted and reorganized for transport on outbound trucks.
The scenario can also arise as a subtask in Automated Storage and Retrieval Systems (AS/RS)~\cite{asrs-survey, yalcin2017multi}, where goods must be resequenced as they arrive from (longer-term) storage for robotic palletization or further transport~\cite{carton-seq}.
Space (footprint-wise) and time-efficient algorithms for routing loads are required as the alternative options involve coupling system components (thereby slowing down delivery) or a lot of ``buffer'' space~\cite{carton-seq}.

\begin{figure}[t]
    \centering
    \includegraphics[width=\columnwidth]{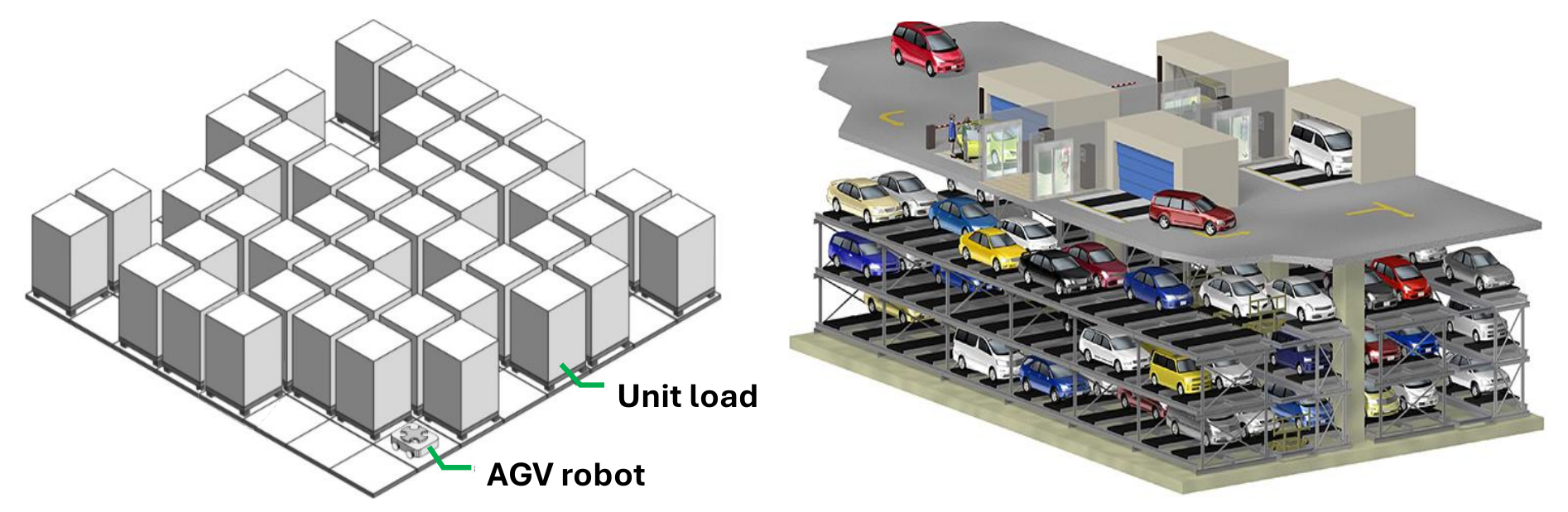}
    \caption{Application examples. Left: Grid-based storage using robotic vehicles (AGVs) for transferring loads~\cite{yalcin2017multi}. The AGV can go beneath a load and can move in all four grid directions. Right: Illustration of an automated parking garage where vehicles are the load to be autonomously placed and retrieved \cite{auto-garage}. Similar to the first case, AGVs can go under vehicles to transport them.}
        \label{fig:motivation}
\end{figure}

A central challenge relevant to the design and operation of the aforementioned environments is determining how to temporarily store items to minimize time, energy, and space utilization.
A fundamental trade-off arises between storage \emph{density}, i.e., the percentage of space used for storage (as opposed to space for access, such as aisles), and rearrangement costs of storage and retrieval, which increase in high-density storage due to limited accessibility to loads.
Each relocation of loads, requiring time-consuming pick-up and drop-off operations, incurs non-trivial additional costs.
Adding to the challenge, information on the arrival and departure time/order of individual loads can be limited or even unknown~\cite{DBLP:conf/esa/BrinkZ14}. %

This work investigates the coupling between storage density and rearrangement efforts by asking the following question: How far can we push to maximize storage space utilization while simultaneously minimizing load rearrangement?
We focus on automated grid-based storage systems, akin to Puzzle-Based Storage (PBS)~\cite{gue2007puzzle}, which are common as loads are typically uniform-sized.
The setup, mirroring practical settings \cite{wurman2008coordinating,JJ-parking}, corresponds to a 2D grid-based storage where each grid cell can store a load. Loads can be moved along the grid in cardinal directions by a mobile robot (e.g., an AGV) provided that they do not collide with other robots/loads. 

\begin{figure*}[ht]
    \centering   
    \begin{overpic}
    [width=\linewidth]{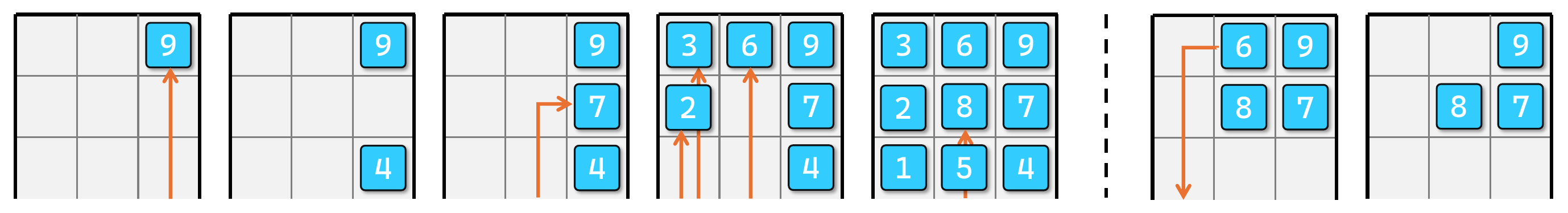}
    \put(6.1, -1.3){{\small ($a$)}}
    \put(19.9, -1){{\small ($b$)}}
    \put(33.5, -1){{\small ($c$)}}
    \put(47., -1){{\small ($d$)}}
    \put(60.8, -1){{\small ($e$)}}
    \put(78.5, -1){{\small ($f$)}}
    \put(92.3, -1){{\small ($g$)}}
    \end{overpic}
    \vspace{-1mm}
    \caption{A solution without relocations for an input arrival sequence $\inc = (9,4,7,3,6,2,1,8,5)$ and departure sequence $\dep = (1,2,3,4,5,6,7,8,9)$ for a $(3 \times 3)$ grid accessible only from the bottom.
    Snapshots are illustrated from left to right. (a) The first arriving (load) 9 can be directly stored at the top using a (straight) upward path empty. (b) Next, 4 arrives and can be stored in front of 9, leaving the space in between. (c) 7 is stored using an upward path with a single turn, i.e., a column-adjacent path. (d) 3, 6, 2, can be stored as shown using upward paths. (e) 1, 8, 5 can be stored similarly. At this stage, all loads have arrived. (f) 1, 2, 3, 4, and 5 can be retrieved sequentially directly using downward paths. Then, 6 can be retrieved using a downward path with a turn, another column-adjacent path. (g) 7, 8, and 9 can be retrieved using downward paths.}
    \label{fig:intro_ex} 
\end{figure*}
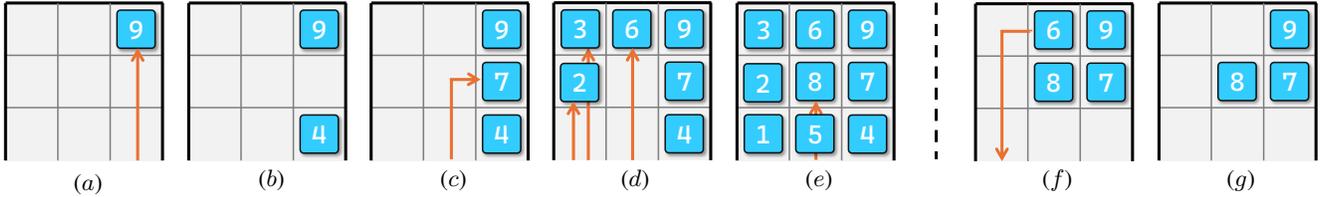

To properly introduce our results (more formal problem statements to follow in \Cref{sec:defs}), the storage area \storage is a rectangular grid with \rows rows and \cols columns, accessible only via the bottom size (see, e.g., Fig.~\ref{fig:intro_ex}).
\storage stores \nobjs labeled (i.e., distinguishable) loads, each occupying one grid cell.
A mobile robot/manipulator can perform the following \emph{actions}: \emph{storage} of an arriving load, \emph{retrieval} of a departing load, or \emph{relocation} (\emph{rearrangement}) of a load from one cell to another, always following a collision-free path through currently empty cells.
Loads arrive sequentially in the order ${\inc = (a_1, \ldots, a_n)}$, where each load needs to be stored in $\storage$ before the next arrival.
Similarly, loads must be later retrieved via another sequence $\dep = (d_1, \ldots, d_n)$.
The objective is to minimize the total number of actions by avoiding rearranging loads.
In the best case, $2n$ actions (a storage action and a retrieval action for each load) are necessary.
A problem/solution is \emph{offline} when \inc and \dep are known in advance. 
Otherwise, the problem is \emph{online}, where two versions are examined.
The first assumes a \emph{lookahead} $\ell$, where only the $\ell$ next arriving loads are known, along with their positions in the departure sequence $\dep$.
In the second, \emph{fully online} setting, no knowledge of \inc and \dep is assumed, other than the number of loads \nobjs.

\textbf{Contributions.}

Our key insight is that the density-rearrangement tradeoff can be (nearly) eliminated given prior information on the arrival and departure sequences.
The precise contributions are as follows:
\begin{itemize}[leftmargin=4mm]
\item For the offline setting, relocation may be required (see \Cref{fig:2x2}) when the number of columns $\cols \le 2$. Surprisingly, for any $\cols \ge 3$, however, it is \emph{always} possible to avoid relocation, for any number of loads, including for storage at full capacity. %
The paper also provides an $O(\nobjs)$-time algorithm that determines how to store \nobjs loads in \storage to solve the problem without needing relocation for $\cols \ge 3$.
\item For the online setting, knowing the full arrival sequence is not necessary. A lookahead of $3\rows -1$ (for $\rows$ rows) suffices for running the presented zero-relocation algorithm for $\cols \ge 3$. For the minimum possible lookahead of 1, it remains possible to find (near-)optimal solutions:
\begin{itemize}[leftmargin=4mm]
    \item For $n \le \rows(\cols-1) +1$, i.e., when it is possible to keep a single column nearly empty, there is an optimal solution that achieves no relocations.
    \item If $\cols \le \rows$, we propose a 1.125-approximate minimization, i.e., a solution where the number of actions is within 1.125 of the optimum for any storage density.
\end{itemize}
\item For the fully online setting, however, one has to sacrifice density to reduce rearrangement. We make a precise characterization and determine the maximum achievable density for a given bound on the number of actions allowed per load.
In particular, we show that if one wishes to guarantee no relocations (i.e., a single action for each storage/retrieval), we must use a very natural, ubiquitous aisle-based storage layout (see \Cref{fig:6x6_depth0}, left) with a maximum density of $2/3$.
    
\end{itemize}

Besides strong guarantees on minimizing relocations during high-density storage and retrieval, our algorithms also deliver desirable paths for storage access. In almost all the cases above, each path taken to store/retrieve a load lies in one column with a possible additional short lateral segment to a cell adjacent to that column, making each motion \emph{distance-optimal} up to an additive factor of 1.
We refer to these paths as \textbf{column-adjacent}.
Beyond speeding up access by limiting distance and turns, such paths enable the use of a straight-access manipulator with only a limited lateral reach from outside the storage area.
This includes the recently introduced robot-arm mounted "bluction" tool \cite{bluction} for accessing objects on shelves. %
The access paths are also naturally suitable for multi-robot execution, as will be discussed in the conclusion.

\section{Related work}

A sizable body of work considers grid-based storage systems with uniform-sized loads.
\citet{gue2006very} introduces an algorithm to design the layout of a rectangular grid-based warehouse given a \emph{depth} bound--the number of objects that may be in the way of another object.
\citet{gue2007puzzle} presents the "Puzzle-Based Storage" (PBS) concept where items are stored on a high-density 2D grid, potentially with only one empty cell called an \emph{escort}. An item can only be moved to a neighboring escort cell, i.e., retrieval is facilitated by repeatedly moving escorts. \citet{gue2007puzzle}~empirically compare average retrieval distances of dense PBS warehouses to traditional aisle-based ones. %
Various extensions to the PBS concept have been proposed, using sequential or parallel motions to move items to/from designated input/output ports; see~\cite{bukchin2022comprehensive} and references within.
Another distinguishing aspect in PBS literature is the objective function considered, e.g., retrieval time or total distance traveled by items~\cite{bukchin2022comprehensive}.
Optimal retrieval policies for the sequential case have been derived for one~\cite{gue2007puzzle} and two escorts~\cite{kota2015retrieval}.
Related to warehouse layout design, \cite{JJ-WCS} investigates the maximization of the number of items stored in a graph-based environment such that all the items are accessible from a connected set of empty vertices.

\citet{carton-seq} study a closely related 2D grid-based load sorting system, with loads arriving in a random order and departing in a specified sequence. %
They present decentralized sorting policies, where each grid cell is an independent unit that can move a load.
As opposed to this work, they focus on empirical performance, have a different access type to the grid, namely, a single input and single output cell, and always reserve one row and one column only for motion.

Sometimes, loads arrive at storage with unknown departure order~\cite{DBLP:conf/esa/BrinkZ14}; items can be rearranged once the departure order is known, provided sufficient downtime time exists, to allow for fast retrieval later.
Finding such a rearrangement with the lowest cost is commonly studied as the Container Pre-marshalling Problem~\cite{caserta2020container}.
Alternatively, the Container Relocation Problem (also called Blocks Relocation Problem) asks to remove all the items in a given ordering while minimizing item relocations~\cite{brp-survey}.
The online variant of the Container Relocation Problem, where the items to be retrieved are revealed one a time, was studied by \citet{DBLP:journals/eor/ZehendnerFJ17}, where the competitive ratio and empirical performance of a leveling heuristic is analyzed.
A different extension is the Dynamic Container Relocation Problem (DCRP)~\cite{hakan2014mathematical}, which also considers the arrival of containers and assumes that both the arrival and departure times of all containers are known (allowing for arrivals and departures to be interleaved).
Although similar to our work in spirit, the aforementioned literature focuses on stacks of loads that are only accessible from the top via a crane, akin to shipyards.
Consequently, the arising minimization problems give rise to a different structure, as evident by their NP-hardness~\cite{brp-survey, caserta2020container}. %

Another closely related effort considers a train yard modeled by a tree whose root is the entrance/exit to the yard, with a known arrival and departure order of trains~\cite{train-ordering}.
They consider the problem of parking the trains as they arrive without relocating them for departure, showing that it is NP-hard, and provide novel approaches to determine feasibility. Here too, motion is more restricted than in our setting, since the dead-end tracks act similarly to LIFO stacks.  %

This work is partly inspired by \emph{mechanical search} research,
which is the problem of retrieving target objects whose position may not be known~\cite{dogar2014object-mech-search}.
Shelves are a closely related type of storage area, as they are also only accessible from one side and give rise to mechanical search due to reduced visibility; e.g., objects may occlude one another~\cite{bluction, huang2021mechanical}.
\citet{KGoptarr} consider an arrangement optimization problem for a grid-based shelf: given a set of objects with associated access costs and access frequencies, optimize the objects' initial placements to minimize access time.

\section{Problem Definition, Notation, Terminology}
\label{sec:defs}

Consider a rectangular $\rows \times \cols$ grid storage space \storage with \rows rows and \cols columns. Fix the orientation of \storage so that its bottom row, also called \emph{front} row, is the open side of \storage through which loads are stored and retrieved.
Denote the columns by $C_1, \ldots, C_{\cols}$ in a left to right order. The loads have distinct labels $1,\ldots, \nobjs$, with $\nobjs \le \rows\cols$. The \emph{density} of a storage space having $\nobjs$ loads is $\nobjs/(\rows \cols)$.
Each load occupies exactly one grid cell. An {\em arrangement} \arr of the loads is an injective mapping $\{1,\ldots,\nobjs\} \to \{1, \ldots, \rows\} \times \{1, \ldots, \cols\}$ that specifies a cell in \storage, i.e., an arrangement is specified by a pair (row, column) for each load. Two loads are \emph{adjacent} in a given arrangement if they are located in horizontally or vertically adjacent grid cells.

Each load can be moved by a robot only via a path of empty cells along the four cardinal directions (up, down, left, or right). Specifically, the following types of \emph{actions} are valid:
\begin{itemize}[leftmargin=4mm]
    \item Storage: A robot can \emph{store} a load in an empty cell $v$ in \storage via a path from any cell in the front row to $v$.

    \item Retrieval: A robot can \emph{retrieve} a load from \storage via a path from its current cell to any cell in the front row.

    \item Relocation: A robot can \emph{relocate} a load within \storage to an empty cell $v$ via a path from its current cell to $v$. The empty call $v$ and the original cell of load must be accessible from the front row. 
\end{itemize}

Loads must be all stored and then retrieved according to prescribed sequences. The departure sequence, i.e., the order in which loads are to be retrieved, is fixed to be ${\dep = (1, \ldots, \nobjs)}$ without loss of generality, as loads can always be relabeled. Denote the arrival sequence, i.e., the order in which loads must be stored, by ${\inc = (a_1, \ldots, a_{\nobjs})}$.
The term \emph{access paths} denotes the paths used to store/retrieve a load.

\smallskip

\noindent \textbf{Problem:} Storage and Retrieval without Relocations ({\tt StoRe$^2$}) \textbf{Input:} A storage area $\storage$ with \rows rows and \cols columns, and arrival and departure sequences \inc and \dep, respectively. \textbf{Question:} Is there a sequence of $2\nobjs$ actions that first stores all the loads per the sequence \inc and then retrieves them per the sequence \dep without any relocations?
\smallskip

We also consider the minimization version of the problem, where relocations may be necessary, where the goal is to minimize the total number of actions for the same input. 

\textbf{Minimization Variation:} Storage and Retrieval with Minimal Relocations ({\tt StoRMR})

In addition to minimizing relocation actions, this manuscript also examines the \emph{total distance} traveled by the loads carried by the robot, which is the sum of the lengths of the paths corresponding to the actions.

\textbf{Online variants.}
In practice, it may not be possible to know the full arrival sequence \inc in advance. Therefore the following two online settings are considered:
\begin{itemize}[leftmargin=4mm]
    \item
    \textbf{Online with lookahead \look}: Only the next \look arriving loads in \inc are known in advance. However, \dep is fully known, meaning the final position of each arriving load in the departure sequence is available.

    \item
    \textbf{Fully online}: Neither \inc nor \dep are known in advance, but the number of loads $n$ is known.

\end{itemize}

The presentation of the developed findings goes from the case of available information to reduced information requirements, i.e., (1) the offline setting first, (2) the online setting with lookahead and, finally, (3) the fully online setting.

\begin{figure}[tb]
    \centering    \includegraphics[width=\columnwidth]{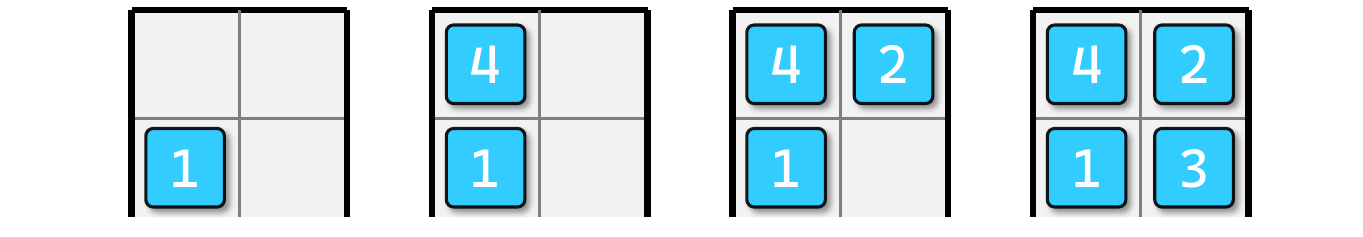}
    \caption{Consider an instance with $A = (1, 4, 2, 3)$ and ${D = (1, 2, 3, 4)}$. Given that load $1$ must depart first, it has to be stored in the front to avoid relocations. This forces the above-shown storage sequence (or its vertical mirror, where 1 is placed on the bottom right). This leaves load 2 buried behind loads 3 and 4. This means it cannot be retrieved without a rearrangement.}
    \label{fig:2x2}
\end{figure}

\section{The offline setting}
This section presents a complete characterization of the existence of rearrangement-free solutions, i.e., for the {\tt StoRe$^2$} problem, in the offline setting. 
In this setting, relocations may be necessary to support $100\%$ density for narrow grid openings:
\begin{observation}\label{o:1}
For a $2 \times 2$ storage space relocations may be necessary.
\end{observation}
\begin{proof}
    Consider the sequence $\inc = (1,4,2,3)$ and recall that the departing sequence is fixed w.l.o.g. to be $\dep = (1,2,3,4)$. Clearly, load 1 must be stored on the front row because it departs first, say w.l.o.g. on the left column (see Fig.~\ref{fig:2x2}). As load 4 arrives, we must place it behind load 1 to not block later arriving loads. This leaves an empty column, which acts like a stack, i.e., we must place load 2 in the back row and load 3 in the front. Finally, observe that the retrieval of load 2 requires relocating either of the loads 3 or 4.
\end{proof}

\begin{figure*}[ht] %
    \centering   
    \begin{overpic}
    [width=\linewidth]{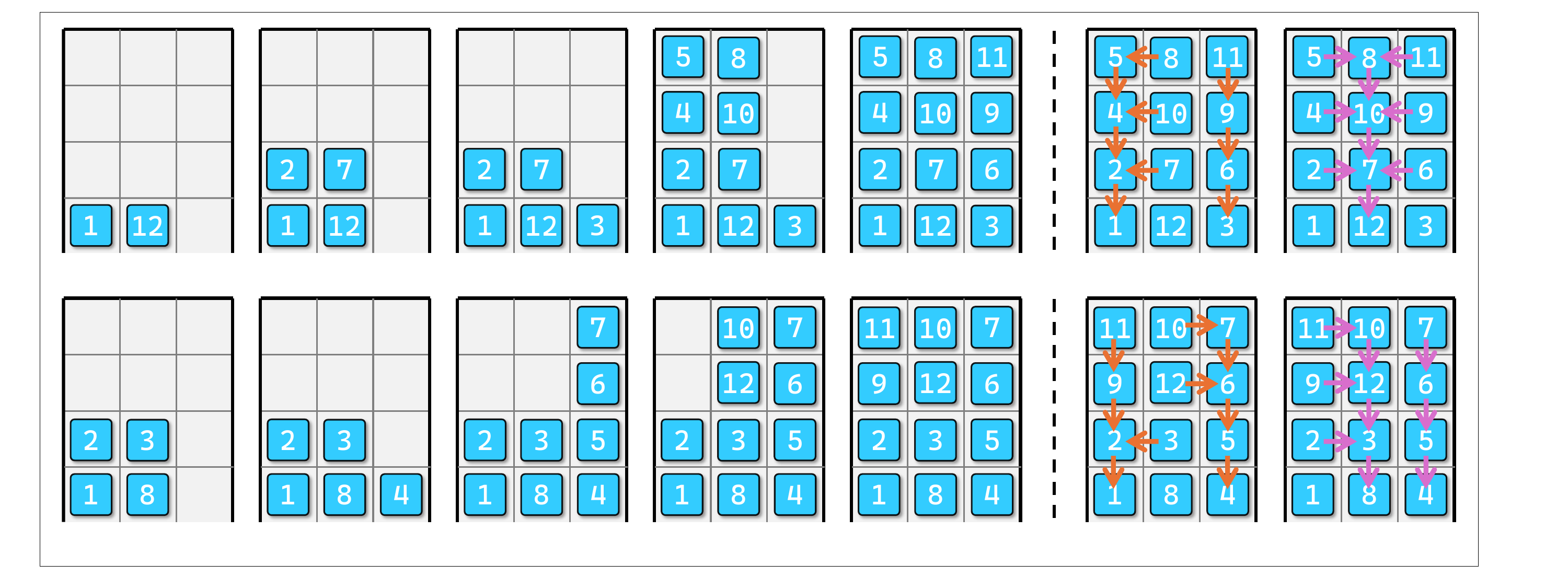}
    \put(6.4, 19.7){{\small ($a$)}}
    \put(20.1, 19.7){{\small ($b$)}}
    \put(34., 19.7){{\small ($c$)}}
    \put(47.4, 19.7){{\small ($d$)}}
    \put(61.3, 19.7){{\small ($e$)}}
    \put(77.4, 19.7){{\small ($f$)}}
    \put(91.6, 19.7){{\small ($g$)}}
    \put(6.4, 0.75){{\small ($h$)}}
    \put(20.1, 0.75){{\small ($i$)}}
    \put(34., 0.75){{\small ($j$)}}
    \put(47.4, 0.75){{\small ($k$)}}
    \put(61.3, 0.75){{\small ($\ell$)}}
    \put(77.4, 0.75){{\small ($m$)}}
    \put(91.6, 0.75){{\small ($n$)}}
    \end{overpic}
    \caption{Running Algorithm~1 on the inputs $\dep' = 
    (\underset{1}{12}, \underset{2}{7}, \underset{3}{3}, \underset{4}{1}, \underset{5}{10}, \underset{6}{8}, \underset{7}{9}, \underset{8}{11}, \underset{9}{6}, \underset{10}{4}, \underset{11}{2}, \underset{12}{5})$
    (top row) and  ${\dep' =
    (\underset{1}{8}, \underset{2}{3}, \underset{3}{4}, \underset{4}{5}, \underset{5}{6}, \underset{6}{2}, \underset{7}{7}, \underset{8}{12}, \underset{9}{1}, \underset{10}{10}, \underset{11}{9}, \underset{12}{11})}$
    (bottom row).
    \vspace{0.05in}\\
    \textbf{Top}: (a)(b) Partial solutions after the first two iterations, where non-equal matching loads of $D$ and $D'$ are put in the left two columns $C_1$ and $C_2$. (c) In the third interaction, $D$ and $D'$ have an equal matching load 3, which is put in $C_3$. (d) The next two iterations fill up columns $C_1$ and $C_2$. This leads to the solution following case 1 from here on. (e) The leftover loads in $C_3$ are filled up to complete the arrangement. The solution with arrows showing the guaranteed local adjacencies for $[12]$ and $\dep'$ is shown in (f) and (g), respectively.
    \vspace{0.05in}\\
    \textbf{Bottom}: (h)(i) A partial solution after two and three iterations. (j) The next three iterations fill up $C_3$, leading the algorithm into case 2. (k) The next two iterations fill up $C_2$ with loads 10 and 12. (l) $C_1$ is filled up to complete the arrangement. (m) and (n) show the solution with local adjacency as in the top row. Notice that the arrows indicate guaranteed paths for retrieval (in (f) and (m)) and for storage (in (g) and (n), when reversed. Better paths requiring fewer sideway-move actions may exist.
    } 
    \label{fig:alg_ex}
\end{figure*}

The observation extends to any $\rows \times 2$ case where $\rows \ge 3$. 

An arrangement \arr is defined to \emph{satisfy} an arrival (resp. departure) sequence \inc  (resp. \dep), if all the loads can be stored (resp. retrieved) according to sequence \inc (resp. \dep) with one action per load, where \arr is the final (resp. initial) arrangement. Fix ${\dep = [n]}$, where $[\nobjs] \coloneqq (1, \ldots, \nobjs)$, leaving $\inc$ as the sole sequence that can change per problem instance. %

\begin{observation} \label{obs:reverse}
An arrangement \arr satisfies an arrival sequence \inc if and only if \arr satisfies the departure sequence in which \inc is reversed.
\end{observation}
\begin{proof}
    All the actions and corresponding access paths that bring the items into the arrangement \arr in the order \inc can be applied in reverse to retrieve the items in the reverse order, and vice versa.
\end{proof}

Following \Cref{obs:reverse}, it suffices to treat {\tt StoRE$^2$} as the equivalent problem of finding an arrangement \arr that satisfies two departure sequences, namely the true departure order $[n]$ as well as a permutation $\dep'$ of $[n]$, where $\dep'$ is the reverse of $\inc$ in the original input. For example, given $A = (1, 4, 2, 3)$ and $D = (1, 2, 3, 4)$ as in Fig.~\ref{fig:2x2}, the two departure orders to be satisfied are the original one $D = (1, 2, 3, 4)$ and $D' = (3, 2, 4, 1)$, which is the inverse of $A$. 

Consequently, it suffices to check the following \textit{local adjacency conditions} to determine whether a given arrangement \arr satisfies a departure sequence $\dep'$:
\begin{observation}
\label{obs:local}
    An arrangement \arr satisfies a departure sequence $\dep'=(d_1, \ldots, d_{\nobjs})$ if and only if every load $d_i$ is either in the bottom row or is adjacent in \arr to an load $d_j$ that departs earlier, i.e., $j < i$.
\end{observation}

Assume the worst case of a maximum density of $100\%$ (otherwise, the situation is simpler). Let $\dep'$ be an input permutation of $[\nobjs]$. The following algorithm constructs a valid arrangement \arr by iteratively assigning placements bottom-up, ensuring the local adjacency conditions of \Cref{obs:local} are met for sequences $[n]$ and $\dep'$.     Recall that $C_1,C_2,C_3$ are the columns in left-to-right order.
    
\textbf{Algorithm 1.} In the first stage, the algorithm iterates jointly over $[n]$ and $\dep'$ in order and repeats the following: Let $x$ and $y$ be the first two items in $[n]$ and $\dep'$, respectively, which have not been assigned a cell. %
If $x \neq y$, assign $x$ and $y$ to the lowest unassigned cells of $C_1$ and $C_2$, respectively. Otherwise, if $x=y$, assign $x$ to the lowest unassigned cell in $C_3$. Repeat this process until one (or more) of the columns becomes fully assigned, at which point we proceed to the next stage, which has two cases:

\emph{Case 1:} If $C_1$ and $C_2$ are fully assigned (notice that since we assign items to them together, they become fully assigned together), we continue filling up $C_3$ (bottom-up) by adding the remaining unassigned items of $[n]$ (in order).

\emph{Case 2:} Otherwise, $C_3$ becomes fully assigned first. In this case, we fill in $C_2$ (bottom-up) with the remaining items of $\dep'$ until $C_2$ is filled up. Then, we fill $C_1$ similarly with the remaining items of $[n]$.

\begin{theorem} \label{thm:3col}
For an $\rows \times 3$ storage area \storage, $\rows \ge 1$, there is a solution that avoids relocations.
The solution can be found in $O(\nobjs)$ time, where $\nobjs$ is the number of loads.
\end{theorem} 
    
\begin{proof} The proof is based on the above algorithm. It verifies that the local adjacency conditions hold for each non-bottom item, when it is assigned a cell in \arr. 

For the first phase and any item $x$ in $C_1$, the adjacent item $z$ below it (resp. $y$ to its right) appears before $x$ in $[n]$ (resp. in  $\dep'$), which follows from the order of placement. Thus, the adjacency conditions hold for $x$. A similar argument follows for $C_2$. Finally, notice that for any item $w$ in $C_3$, the adjacent item $z$ below it must appear before $w$ in both $[n]$ and $\dep'$.
    
Let us examine the second phase:

\emph{Case 1:} For an item $w$ placed in this phase, the adjacent item $z$ on its left (resp. $z'$ below it) appears before $w$ in $\dep'$ (resp. in  $[n]$).

\emph{Case 2:} %
For an item $w$ added to $C_2$, the adjacent item $z$ on its right appears before $w$ in both $[n]$ and $\dep'$. %
When an item $w$ is added to $C_1$, the adjacent item $z$ on its right (resp. $z'$ below it) appears before $w$ in $\dep'$ (resp. in  $[n]$).

Finally, it is straightforward to verify that the algorithm runs in $O(n)$ time.    
\end{proof}

We can now use \Cref{thm:3col} as a building block to obtain a solution without relocation for storage spaces with more than three columns. The proof of \Cref{thm:char} will be provided when we prove \Cref{thm:general-full}, a more general result. 

\begin{theorem}\label{thm:char}
For an $\rows \times \cols$ storage area \storage with $\rows > 1$, there is a solution that avoids relocations if and only if $\cols \ge 3$.
\end{theorem}

\section{Online setting with lookahead}

This section analyzes the variant where the position of each arriving load in the departure sequence \dep is known but only a limited lookahead is available for the arrival sequence \inc. Practical considerations for the presented storage placement strategies are also discussed.

Unlike the offline case of \Cref{thm:3col}, there is no longer access to the reversed version of $\inc$, so the following analysis works with the original arrival sequence $\inc$. Consider now a generalization of \Cref{thm:3col} and the proof of Theorem~\ref{thm:char} via the following more general statement. 

\begin{theorem} \label{thm:general-full}
For an \( \rows \times \cols \), \( \cols \ge 3\) workspace, with a lookahead of \( \ell = 3\rows -1\), loads can be stored up to the maximum capacity without relocations in $O(\nobjs \log \rows)$ time.
Furthermore, this can be done with column-adjacent access paths.
\end{theorem}\vspace{-0.1in}
\begin{proof}
Assume a density of 1; otherwise, the situation is simpler.
The argument proceeds in two stages.
First, iterate over the $\cols - 3$ leftmost columns in left-to-right order, filling up a column at a time as follows (see \Cref{fig:t34} for an example): 
Let $C$ denote the current column, and let ${s = (a_{(1)}, a_{(2)}, \ldots, a_{(\rows)})}$ denote the next $\rows$ loads in $\inc$ sorted per departure order.
Notice that we can determine $s$ using a lookahead of $\rows < \ell$.
We proceed to place each load $a_{(i)}$ in the cell of column $C$ located on the $(i)$-th row away from the front row. %
At this stage, notice that each cell in $C$ is accessible from the next column to the right, which is empty. Consequently, each storage path is column adjacent. The placement also ensures that each load can be retrieved directly using only $C$, as the load below it departs first.

After the first stage, we apply the algorithm from \Cref{thm:3col} to the remaining three columns and inherit the theorem's properties. A lookahead of $3\rows -1$ suffices as the last arriving load can be easily deduced.

Finally, the running time follows from sorting $\rows$ loads at a time for at most $\nobjs/\rows$ batches.
\end{proof}

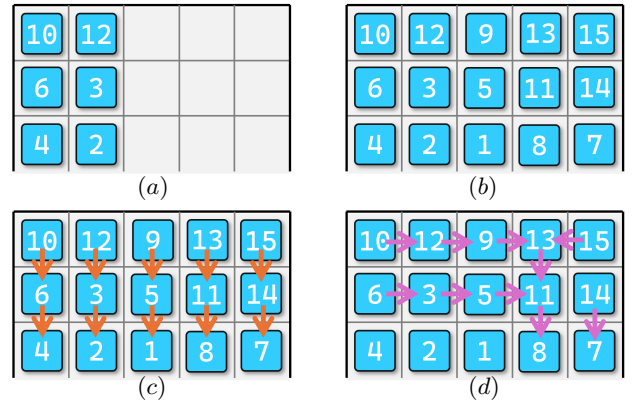
\begin{figure}[b!] %
    \centering
    \begin{overpic}[width=\columnwidth]{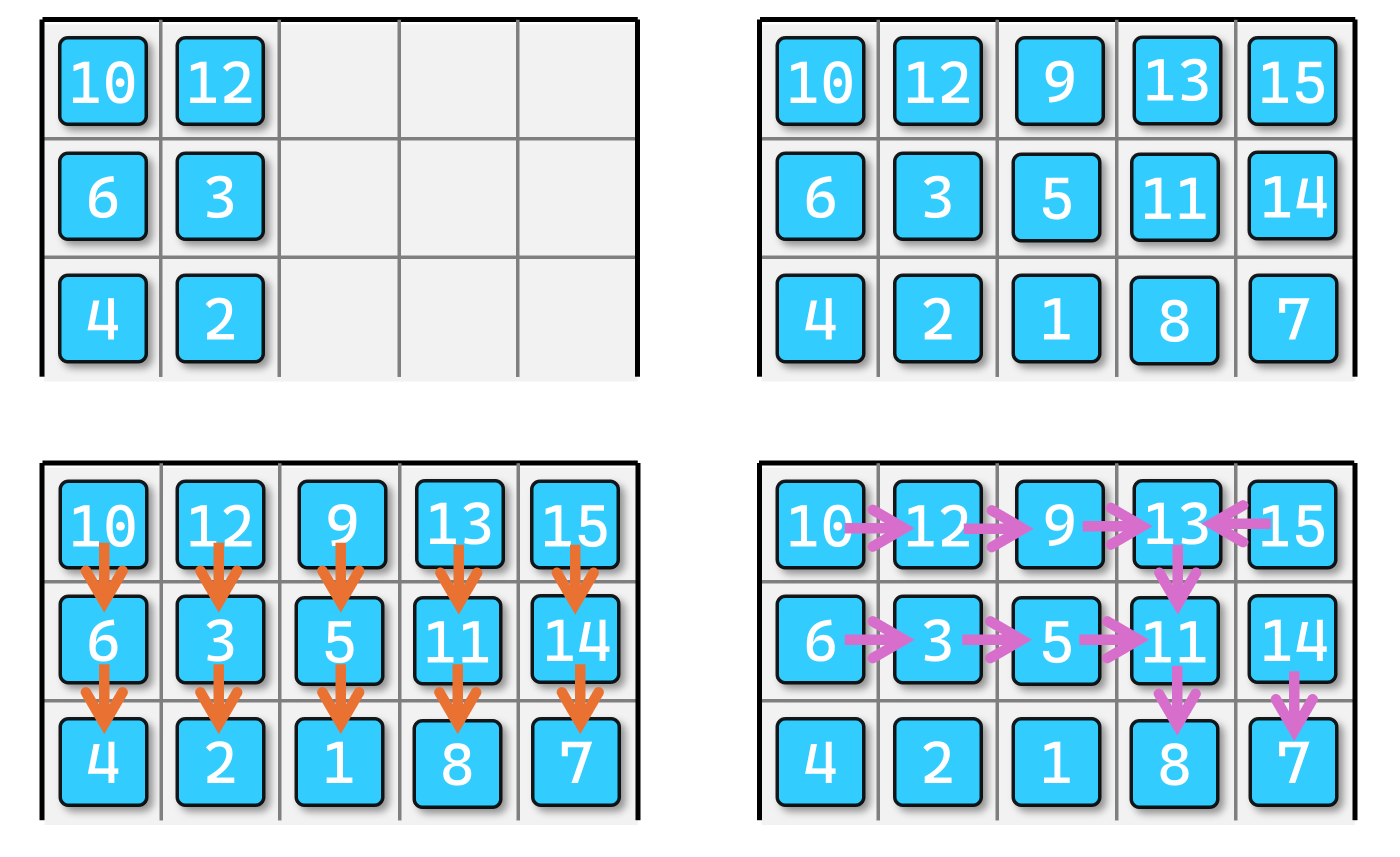}
    \put(22, 30.8){{\small $(a)$}}
    \put(73, 30.8){{\small $(b)$}}
    \put(22, 0){{\small $(c)$}}
    \put(73, 0){{\small $(d)$}}
    \end{overpic}
    \vspace{-3mm}
    \caption{An example execution of the algorithm described in Theorem~\ref{thm:general-full}, which also works for proving Theorem~\ref{thm:char}. $A = (4, 10, 6, 12, 2, 3, 9, 15, 1, 14, 13, 7, 5, 11, 8)$ and $D=[15]$. (a)~Knowing the first three loads to arrive are $4, 10, 6$, we store them in the order they depart, which is $4, 6, 10$ from bottom to top. These go to $C_1$.
    Similarly, the next three arrivals are stored in $C_2$ as $2, 3, 12$, from bottom to top. (b) For the next $9$ arrivals, we run the algorithm from Theorem~\ref{thm:3col}.
    For this, we turn the remaining loads of $A$ backward to get $D'=(8, 11, 5, 7, 13, 14, 1, 15, 9)$ and the corresponding portion of $D$ is $(1, 5, 7, 8, 9, 11, 13, 14, 15)$.
    (c)(d) Arrows showing how departures and arrivals can be handled without rearrangements, respectively. Note that the arrows for arrivals are drawn backwards to be consistent with Fig.~\ref{fig:alg_ex}. 
    }
    \label{fig:t34}
    \vspace{-.15in}
\end{figure}

The proof of Theorem~\ref{thm:general-full} makes it clear that for a known departure order, it is possible to fill up columns as loads arrive without much lookahead, for most of the storage space, while guaranteeing no rearrangement is needed during the retrieval phase. Proposition~\ref{prop:fill-upto-last-col} formalizes this observation.

\newpage
\begin{prop} \label{prop:fill-upto-last-col} %
For an \( \rows \times \cols \) workspace, ${\rows(\cols - 1) + 1}$ (or fewer) loads can always be stored without any relocations and with a lookahead of 1.
Furthermore, this can be done with column-adjacent access paths.
\end{prop}
\begin{proof}
Let us assume there are $\nobjs = \rows(\cols - 1) + 1$ loads, as it is straightforward to handle fewer loads.
The goal is to fill the leftmost $\cols-1$ columns top-down such that at the departure stage, the loads in each column $C_i$ depart after the loads at column $C_{i+1}$ have departed.
We store each load in a column at the current topmost available cell based on its departure order:
Load $1$ is stored at $C_{\cols}$, at its front cell.
Loads $2, \ldots, 2+\rows - 1$ are stored at $C_{\cols-1}$, loads $2+\rows, \ldots, 2+2\rows-1$ are stored at $C_{\cols-2}$, and so on.
This approach results in a straight upward path for each storage action. Each retrieval is possible via a column-adjacent path (via the adjacent column on the right) or a straight path.
\end{proof}

The observation leading to Proposition~\ref{prop:fill-upto-last-col} allows to show that if \storage is square-shaped or wider (at its open side), full-capacity storage with limited relocations is possible.

\begin{theorem} \label{prop:L-paths}
For an \( \rows \times \cols \) workspace with $\rows \le \cols$, loads can always be stored at full capacity (or less) with at most \( \rows - 1 \) relocations,
with one action for each storage and at most two actions for each retrieval,
using a lookahead of 1.
\end{theorem}

\begin{proof}
    We generalize the column-filling approach from \Cref{prop:fill-upto-last-col} to a path-filling approach, i.e., we define a sequence of paths, each of which will contain loads departing only after the loads on the next path.
    The first $\cols-\rows$ paths are the leftmost $\cols-\rows$ columns (we have $\cols-\rows \ge 0$).
    The remaining portion of \storage is an $\rows \times \rows$ square.
    Take the next path to be the square's leftmost column and top row, i.e., an upward segment and a rightward segment.
    Define the next path similarly for the remaining $(\rows-1) \times (\rows-1)$ square, and so on. See \Cref{fig:L-paths} for a visualization of the corresponding notions.
    
    Denote the resulting complete sequence of paths by $\pi_1, \ldots, \pi_{\cols}$.
    Each path contains a cell on the front row, so we can fill each path with arriving loads front to back (similarly to a stack).
    Loads are assigned to the paths based on their departure order:
    Load $1$ is assigned to $\pi_{\cols}$ (which is a single cell), loads $2,3,4$ are assigned to $\pi_{\cols-1}$, and so on, with the last $|\pi_{1}|$ departing loads assigned to $\pi_{1}$.

As loads arrive, they are stored using the paths above, which potentially contain one turn and one action. On departure, each load located only on a vertical segment of some $\pi_i$ (for $i < \cols$) can be retrieved using a column-adjacent path using the column to the right, i.e., via $\pi_{i+1}$, since the loads on $\pi_{i+1}$ depart earlier. Loads located only on a horizontal segment of some $\pi_i$ can be similarly retrieved via a downward path. Both of the latter types of loads require only one action. A load located on a "corner" of some $\pi_i$, however, may not be adjacent to an earlier departing load and hence may require relocation. There are $r-1$ such loads and hence at most $r-1$ relocations. Let $x$ be such a load that is blocked at its departure. We relocate the adjacent load $y$ below $x$ at most two cells to the right, which is possible since $\pi_{i+1}$ and $\pi_{i+2}$ (assuming $i+2 \le \cols$) are empty at this stage. The relocation results in a clear retrieval path for $x$ and allows $y$ to be retrieved via a downward path. Notice that it may also be possible to relocate a blocking load of $x$ one cell right/down for it to be retrieved.  Overall, each load requires at most 2 actions to be retrieved.
\end{proof}

\begin{figure}[tb]
    \centering  
    \begin{overpic}[width=\columnwidth]{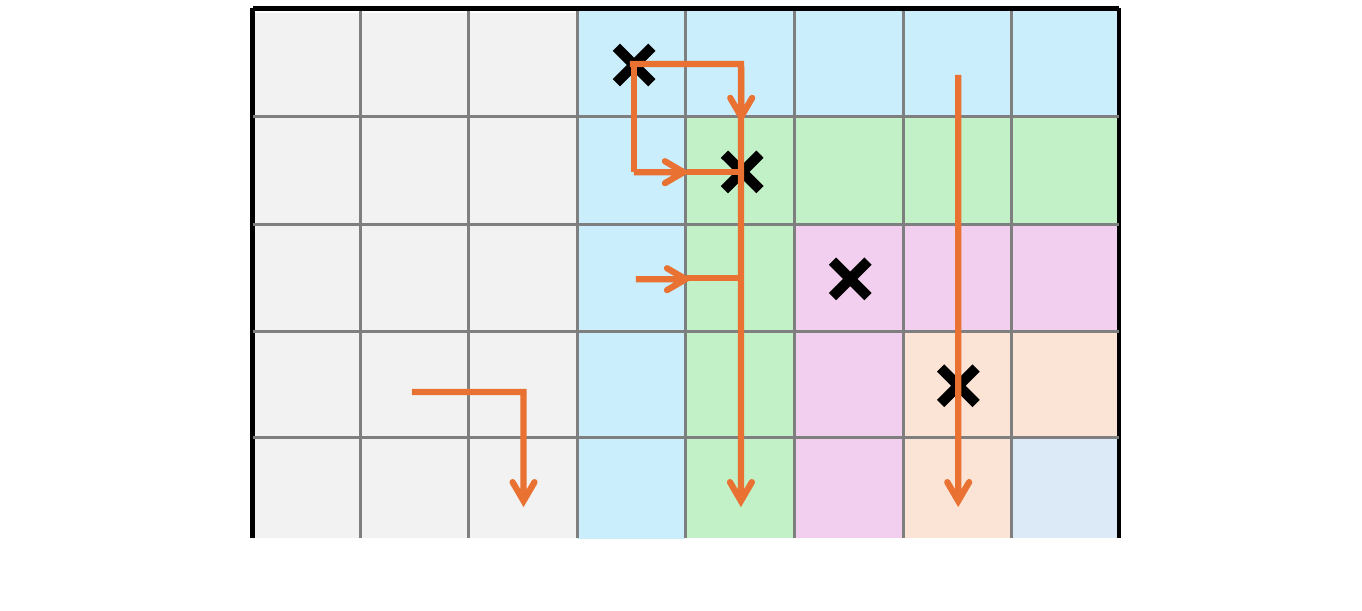}
    \put(22, 1.5){{\small $\pi_1$}}
    \put(29.5, 1.5){{\small $\pi_2$}}
    \put(37, 1.5){{\small $\pi_3$}}
    \put(45, 1.5){{\small $\pi_4$}}
    \put(52.5, 1.5){{\small $\pi_5$}}
    \put(61, 1.5){{\small $\pi_6$}}
    \put(68.5, 1.5){{\small $\pi_7$}}
    \put(76, 1.5){{\small $\pi_8$}}
    \end{overpic}
    \caption{The storage strategy of \Cref{prop:L-paths} for a $(5 \times 8)$ grid.
    Paths $\pi_1$--$\pi_3$ are the left three vertical columns. Paths $\pi_4$--$\pi_7$ are $L$-shaped and distinguished using different colors. $\pi_8$ passes only the bottom right cell. Corner cells are marked with crosses, and the possible types of retrieval paths are shown using arrows, except for loads that can be relocated.}
    \label{fig:L-paths}
\end{figure}
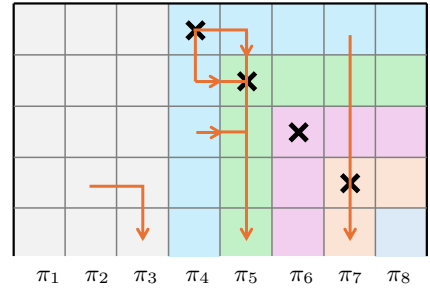

\ifthenelse{\boolean{isfull}}{ %
    We combine our results to achieve a 1.125-approximation for the minimum number of actions:
    
    \begin{corollary} %
    For an \( \rows \times \cols \) workspace with $\rows \le \cols$, the problem of storage and retrieval can be solved with a minimum number of actions in $O(n)$ time with an approximation ratio\footnote{The approximation ratio is less than $1.05$ for $\rows > 9$ and decreases further as density decreases.} of $1.125$  using a lookahead of 1.
    \end{corollary}

    \begin{proof}
    If there are $\nobjs = \rows(\cols - 1) + 1$ or fewer loads, we can apply the strategy of \Cref{prop:fill-upto-last-col} and achieve no relocations.
    Otherwise, suppose $\nobjs = \rows \cols - k$ for some ${0 \le k < r-1}$, i.e., there are exactly $k$ empty cells when we fill \storage.
    In this case, we apply the strategy in \Cref{prop:L-paths} while not using $k$ of the $r-1$ corners.
    Since only a load placed in a corner can result in a relocation during its retrieval, we have at most ${2\nobjs + (\text{\#corners occupied})} = 2\nobjs + \rows -1 -k$  total actions.
    Therefore we get the approximation
    \[
    \frac{2\nobjs + \rows - 1 - k}{2\nobjs} = \frac{2\rows\cols + \rows - 1 - 3k}{2\rows\cols - 2k} \le \frac{2\rows\cols + \rows - 1}{2\rows\cols} \le 1.125\text{,}
    \]
    where the first inequality is straightforward to verify and for the second one we notice that the ratio is maximized for $\rows = \cols = 2$.
    \end{proof}
 }{ %
Using the fact that the solution above uses at most $r-1$ relocations, it is straightforward to verify that the solution provides a $1.125$-approximation, i.e., the number of actions is within $1.125$ of the optimum. The proof is omitted.}

\section{The fully online case}%

In the fully online setting, density typically has to be limited to minimize load storage/retrieval time. This means that the arrangement of the loads in \storage needs to be carefully chosen.  This section analyzes the relationship between storage density and the number of actions required for storage and retrieval.

\subsection{Depth and density}
Recall first a key notion for the number of loads that may potentially block a load that is to be retrieved. An arrangement of loads \arr has {\em depth} $k$ if every load can be retrieved by first relocating/retrieving no more than $k-1$ other loads.

Given a fixed depth $k \in \mathbb{N}$, a key question is to determine the maximum density that can be accommodated while ensuring an arrangement of the loads with depth at most $k$ exists.
Denote this density by $\rho(k)$.
This relationship is investigated by \cite{gue2006very}, which establishes the following:

\begin{theorem}[Upper bound, \cite{gue2006very}] \label{thm:opt-density}
For rectangular 2D grids, $\rho(k) \le 2k/(2k+1)$.   
\end{theorem}

In the setting of \cite{gue2006very}, however, all the free cells are required to form a single connected component containing a single input/output cell.
The proposed setting in this work does not impose such a requirement, as the assumption is that the whole bottom row of \storage can be used to store and retrieve loads. 
Therefore, it is possible to realize the upper bound using the following simple aisle-based approach.

\begin{prop}\label{p:lb}
For a given depth $k$, if the number of columns \cols in \storage is a multiple of $2k+1$, then there is an arrangement that achieves the optimal density of $2k/(2k+1)$.
\end{prop} 
\begin{figure}[tb]
    \centering    \includegraphics[width=0.95\columnwidth]{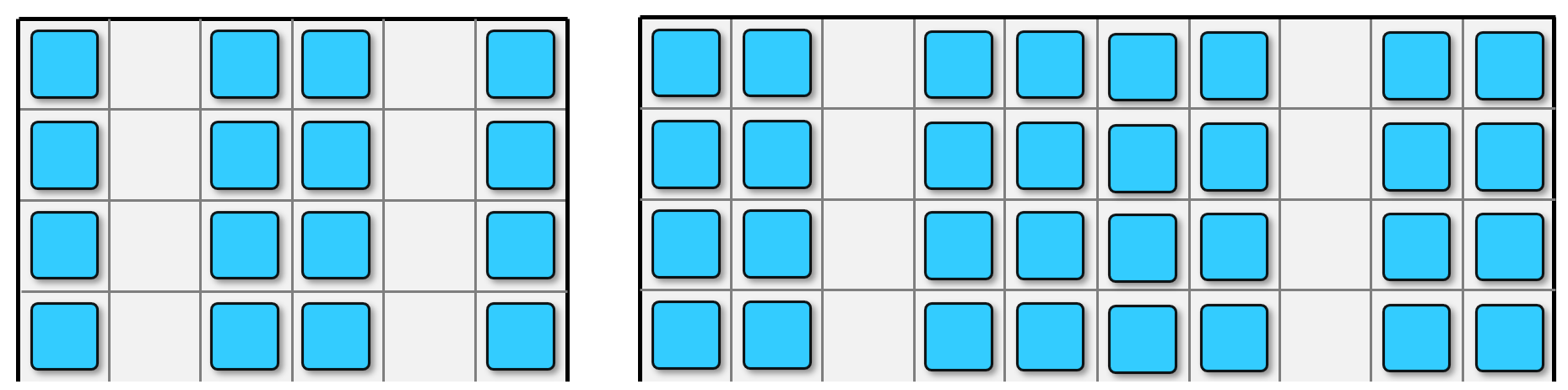}
    \caption{Left: A $4\times 6$ storage area with 1-deep aisles. Right: A $4\times 10$, 2-deep aisle arrangement.
    }
    \label{fig:6x6_depth0}
\end{figure}
\vspace{-4mm}
\begin{proof}
Consider an aisle-based configuration where each grid column is either completely full of loads or empty. For each set of $2k+1$ contiguous columns, starting from the leftmost $2k+1$ columns, set all the columns to be full of loads, except the $(k+1)$th column (i.e., the middle column of the set), which forms an empty aisle. See \Cref{fig:6x6_depth0}. This results in a depth of $k$ since any cell with a load is at a distance of at most $k$ from an empty aisle column, while the density is achieved since for every $2k+1$ columns, there are $2k$ columns full with loads.
\end{proof}

We call the latter pattern a \textbf{$k$-deep aisle} arrangement and call the empty columns in it \emph{aisles}.

\textbf{Remark:}
The case of 1-deep aisles is frequently adapted in practice, e.g., at libraries or grocery stores, where books or goods can be stored and retrieved using a single action via an aisle. Proposition~\ref{p:lb} explains how this natural strategy is effective in balancing storage density and access efficiency by setting $k = 1$ in this particular case. 

\vspace{-3pt}
\subsection{Action bounds and density}
We now analyze the connection of depth and density to the worst-case number of actions in the fully online setting.
We show that the trade-off between the number of actions and density carries over.

\ifthenelse{\boolean{isfull}}{ %
    \begin{figure}[tb]
    \centering
    \begin{tikzpicture}
        \begin{axis}[
            width=0.8\linewidth,
            height=3.7cm,
            xlabel={Number of actions ($a$)},
            ylabel={
                \parbox[c]{2.5cm}{\centering
                    Optimal density:\\ 
                    $2a/(2a+1)$
                }
            },
            ylabel style={font=\small},
            xtick={1,2,3,4,5,6,7,8,9,10},
            ymin=0.6, ymax=1.05,
            ytick={0.6, 0.7, 0.8, 0.9, 1},
            domain=1:10,
            samples=10,
        ]
            \addplot[mark=*, blue] 
            coordinates {
                (1, 2/3)
                (2, 4/5)
                (3, 6/7)
                (4, 8/9)
                (5, 10/11)
                (6, 12/13)
                (7, 14/15)
                (8, 16/17)
                (9, 18/19)
            };
        \end{axis}
    \end{tikzpicture}
    
    \caption{Maximum density versus the maximum number of actions allowed for each storage/retrieval.}
    \label{fig:density-plot}
    \end{figure}
}{ %

}

\begin{prop}
To achieve a given worst-case bound ${a \in \mathbb{N}}$ on the number of actions for each storage/retrieval, the density must be at most $2a/(2a+1)$.
Conversely, a density of ${2a/(2a+1) - \varepsilon}$ is possible for a small $\varepsilon > 0$.
\end{prop}
\begin{proof}
By \Cref{thm:opt-density} we cannot obtain a density greater than $2a/(2a+1)$, as otherwise after storing all the loads, there would be a load that requires more than $a$ actions to be retrieved due to blocking loads.

We can achieve a density that approaches $2a/(2a+1)$ as either (or both) of the grid dimensions tends to infinity.
We use an $a$-deep aisle arrangement \arr, with the slight change of keeping a "buffer" of $a-1$ empty cells, which we take to be any cells adjacent to an aisle (i.e., empty column).
We first store the loads per \arr, which can be easily done using one action per load since we do not assign loads to specific cells.

For retrieval, our goal is to keep each aisle empty after each load is retrieved, which ensures that the depth of any subsequent arrangement remains at most $a$.
To this end, when we retrieve a load, we relocate each load that blocks the way to an empty non-aisle cell.
Since we start with $a-1$ such "buffer" cells, and at most $a-1$ loads are relocated, this can always be done.
\end{proof}

We conclude by observing that without any information on the departure/arrival sequences, one has to be conservative when choosing an appropriate arrangement and sacrifice density to ensure that retrievals can happen within a given number of actions.\ifthenelse{\boolean{isfull}}{ %
See \Cref{fig:density-plot}.
}{}

\vspace{-2pt}
\section{Experimental evaluation}

\begin{figure}[tb]
    \centering  
     \includegraphics[width=0.85\columnwidth]{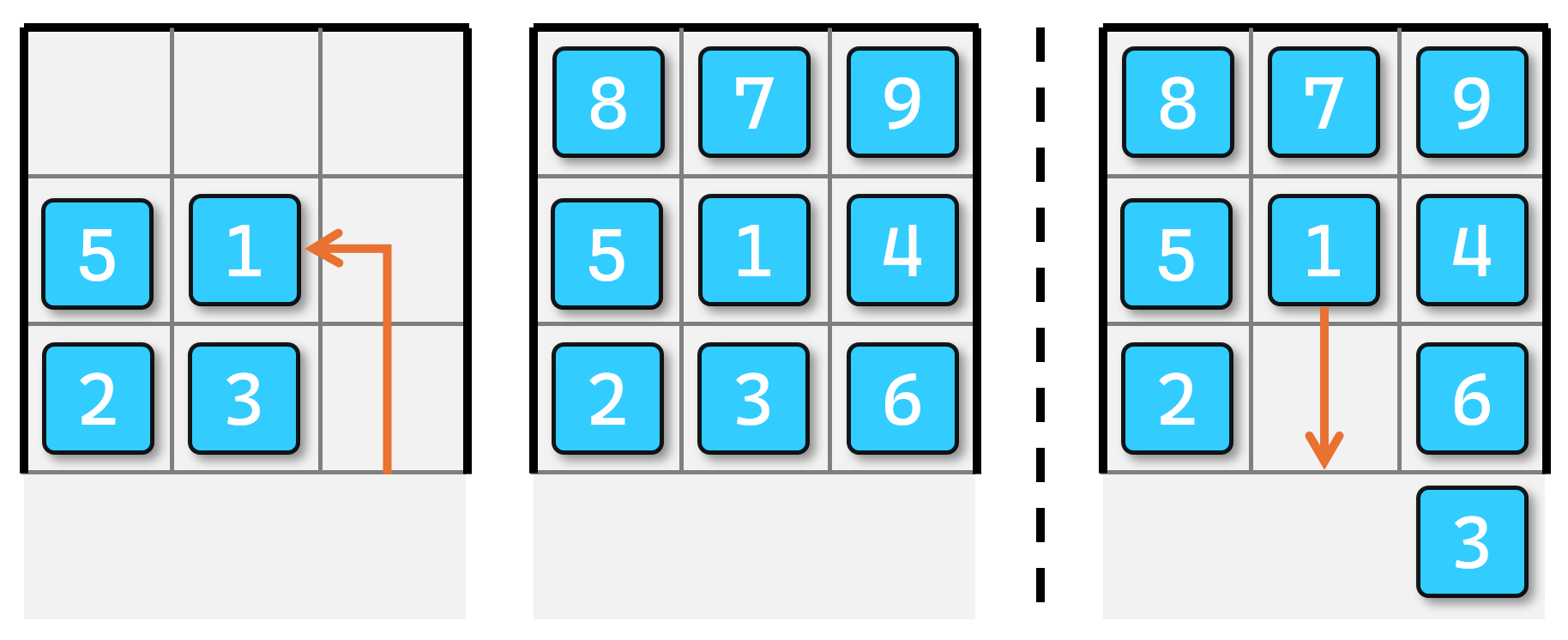}
    \caption{Snapshots of the baseline solution for ${\inc = (5, 2, 3, 1, 8, 7, 9, 4, 6)}$.
    Left: The first three arriving loads are stored in their respective designated rows. When $1$ arrives, storing it in the front row would block access to empty cells, and so it is stored in the next row.
    Middle: The last three loads are stored in the right column in their arrival order.
    Right: Upon retrieval, $3$ is relocated outside of $W$ to let $1$ out. Load $3$ is then stored back in its assigned cell, after which no relocations occur.
    }
    \label{fig:baseline-ex}
\end{figure}

To validate the presented theoretical results, we experimentally compare the main algorithm of the offline setting (from \Cref{thm:general-full}) against a straightforward best-first baseline.
The baseline stores earlier departing loads towards the front, aiming to store the first $c$ departing loads on the front row, the next $c$ departing loads on the second row, and so on.
Each row is filled from left to right without relocations in the storage phase:
If storing a load in its designated row disconnects empty cells from the front row, the load is stored on the next available row such that access to remaining empty cells is maintained.
Distance-optimal paths are used, except that during retrieval, we first minimize the number of blocking loads along the path (and then minimize distance).
As opposed to the proposed algorithm, the baseline may have to temporarily move blocking loads outside of $W$; see \Cref{fig:baseline-ex}.
Each load that blocks the current target load is moved outside of the grid along the retrieval path and (after the target load is retrieved) then stored back in its original cell along the same path.
A supplementary video illustrates solutions of the proposed algorithm and the baseline side by side, with a robot executing the motion.

\begin{table}[b]
\centering
\vspace{1mm}
\begin{tabular}{lccccc}
\toprule
\# rows/columns & $10$ & $15$ & $20$ & $25$ & $30$\\
\midrule
\textbf{\# actions} &  & & & &\\
baseline  & 127 & 283 & 496 & 769 & 1104 \\
{\bf proposed} & {\bf 100} & {\bf 225} & {\bf 400} & {\bf 625} & {\bf 900} \\
baseline suboptimality & 27\% & 26\% & 24\% & 23\% & 23\% \\
\midrule
\textbf{total distance } & & & & &\\
baseline & 1,387 & 4,564 & 10,486 & 20,009 & 34,324 \\
\textbf{proposed} & \textbf{1,170} & \textbf{3,774} & \textbf{8,727} & \textbf{16,779} & \textbf{28,679} \\
baseline suboptimality & 26\% & 27\% & 25\% & 23\% & 23\% \\
\textbf{our suboptimality}& \textbf{6\%} & \textbf{5\%} & \textbf{4\%} & \textbf{3\%} & \textbf{3\%}\\
\bottomrule
\end{tabular}
\vspace{1mm}
\caption{Performance of proposed algorithm vs. the baseline. %
}
\label{tab}
\vspace{-1mm}
\end{table}

We compare the proposed algorithm against the baseline on square grids filled to capacity. We measure actions during retrieval and the total distance traveled by loads, averaged over 25 random instances for each grid size; see \Cref{tab}.
For the baseline, the shown distance considers the motion of loads only inside $W$, as this is the focus of this work. For each cost metric, the suboptimality ratio is defined as the cost incurred divided by the optimum.
For distance, we use a lower bound ($m^3+m^2$ for a grid side length of $m$) in lieu of the optimum.
For both the action and total distance metrics, the baseline is always more than 20\% suboptimal. The proposed algorithm is optimal for relocation actions and at most 6\% suboptimal for total distance, confirming the diminishing impact of short sideways motions. %

\vspace{-4pt}
\section{Discussion and conclusion}
\vspace{-2pt}
This paper analyzes planning the storage and retrieval of uniform loads in high-density 2D grid-based storage while minimizing intermediate load rearrangement.
We establish an intriguing, sharp characterization for the offline case, showing rearrangement can always be avoided if and only if we can access a side of at least 3 cells wide. 
When only the departure sequence is known, we also identify settings where rearrangement can be avoided or the number of actions can be minimized to within a $9/8$ suboptimality bound.
These positive results contrast stack-based or stack-like storage settings, in which related problems are NP-hard~\cite{train-ordering, caserta2020container}.
In this sense, our results, especially \Cref{thm:3col}, are surprising, as at first glance, some of our variants appear NP-hard as well.
From a practical standpoint, our results -- both the proposed dense storage configurations and the associated algorithmic solutions --  are promising for real-world adoption.

Despite the apparent sequential nature of the plans produced by our key algorithmic results, we note that these plans are in fact highly amenable to parallel execution by multiple robots/AGVs. As an illustrative example, in executing plans according to Theorem~\ref{thm:general-full}, it is executed column by column until the last three. For the first $(c - 3)$ columns, loads can be carried in parallel from the right side of the column by multiple robots, which can then exit downward from the column underneath the loads, with minimum delays and conflicts. As the last three columns get filled, after storing a load, robots can move to the top and then left to exit downward from $C_{c-3}$. Similarly, parallel executions of the retrieval process can be carried out by multiple robots. In other words, our algorithmic results are naturally suitable for scalable execution using a fleet of robots, ideal for modern warehouses, ports, and other logistics automation settings. 

Beyond parallel motion, future research could allow storage and retrieval to be interleaved in an arbitrary manner, as opposed to storage strictly followed by retrieval.
Finally, for cases where we require a lookahead greater than 1, it is natural to ask whether a lower lookahead is possible.

\vspace{-1.5mm}
{\small
\bibliographystyle{IEEEtranN}  %
\bibliography{references, mrmp-references}
}

\end{document}